\documentclass[letterpaper,twocolumn,fleqn]{article} 

\usepackage{graphicx}
\DeclareGraphicsExtensions{.pdf}
\usepackage{amsmath}
\usepackage{amsthm}

\usepackage{times}
\usepackage{amssymb}

\usepackage{booktabs}
\usepackage[sort&compress,numbers]{natbib}
\usepackage[ruled,linesnumbered]{algorithm2e}
\newcommand{\norm}[1]{\left\lVert #1 \right\rVert}

\title{Color Homography Color Correction}

\author{Graham D. Finlayson~$^{1}$, Han Gong~$^{1}$, and Robert B. Fisher~$^{2}$\\
$^{1}$University of East Anglia, UK\\
$^{2}$University of Edinburgh, UK}

\date{} % date has an empty field.

\begin{document} 

\maketitle 

\thispagestyle{empty} % prevents the first page to be numbered

%%%%%%%%%%%%%%%%%%%%%%%%%%%%%%%%%%
% Abstract
%%%%%%%%%%%%%%%%%%%%%%%%%%%%%%%%%%

\begin{abstract}
Homographies -- a mathematical formalism for relating image points across different camera viewpoints -- are at the foundations of geometric methods in computer vision and are used in geometric camera calibration, image registration, and stereo vision and other tasks. In this paper, we show the surprising result that colors across a change in viewing condition (changing light color, shading and camera) are also related by a homography. We propose a new color correction method based on color homography. Experiments demonstrate that solving the color homography problem leads to more accurate calibration.
\end{abstract}

%%%%%%%%%%%%%%%%%%%%%%%%%%%%%%%%%%%%
% Overall Document Guidelines: Head
%%%%%%%%%%%%%%%%%%%%%%%%%%%%%%%%%%%%
\section{Introduction}
\label{sec:intro}
In image formation there are two important parts, the geometry of how points in space map to image locations and the photometry of how illumination, surface reflectances and camera sensors combine to form the colors in an image. Broadly, the mathematical tools underlying our understanding of image geometry are non-linear reflecting the non-linear perspective nature of image formation. Important non-linear concepts include ''solving for the homography'' (e.g. relating subsequent frames in panorama stitching~\cite{BL07}) and epipolar geometry in stereo vision~\cite{Hartley2004}). In contrast, the majority of methods in color/photometric computer vision are linear which, at least for simplified scenes such as the eponymous Mondrian world~\cite{LAND.SCIENTIFIC.AMERICAN,FORSYTH.NOVEL.IJCV.90} (the world consists of a patchwork of flat co-planar reflectances), reflects the physics of how images are formed. Linear color problems include, color correction~\cite{VRHELMATHEMATICS,DREW.NAT.MET.91} (e.g. mapping raw colors from camera to display RGB) and modeling illuminant color change~\cite{WANDELL87} e.g. for color object recognition~\cite{Lenz}.

\begin{figure}
\begin{center}
  \includegraphics[width=\linewidth]{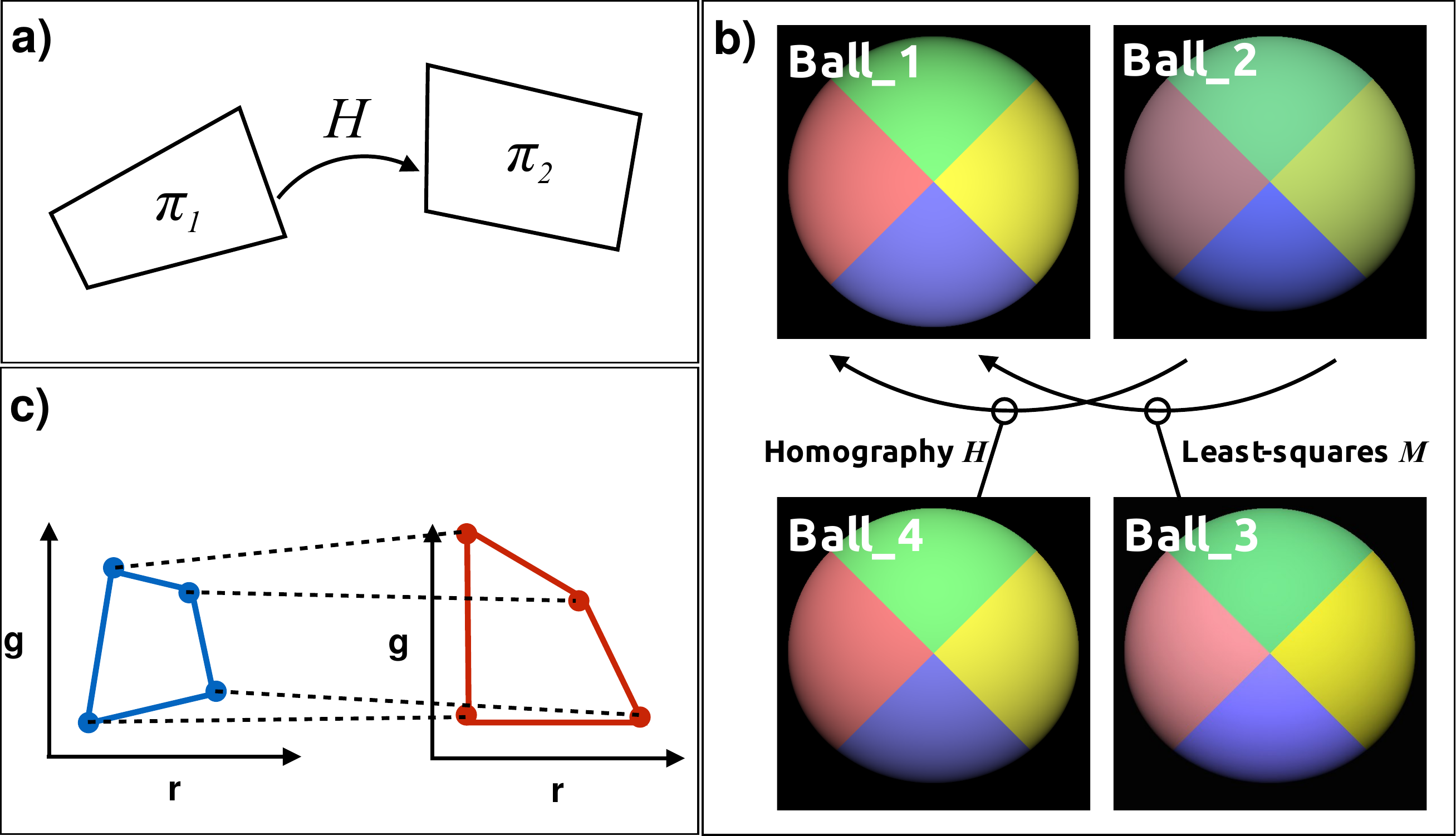}
\end{center}
\caption{Top left, panel (a), images of two planes are related by an homography. Right, panel (b), 4 images of a colored ball are shown. Ball\_1 is the reference image where the illumination color is white and placed behind the camera. Ball\_2 is the object illuminated with a blue light from above. Respectively, Ball\_3 and Ball\_4 are the least-squares mapping and the homography match from Ball\_2 to Ball\_1, Bottom right, panel (c), the chromaticities from Ball\_2 matched to corresponding chromaticities in Ball\_1.}
\label{fig:balls}
\end{figure}

In geometry computer vision, an homography relates two planes. In Figure~\ref{fig:balls}a, $\pi_1$ and $\pi_2$ might denote the same 3D plane viewed in two different images related by the homography $H$.
In color, an homography relates two {\it photometric} views. In Figure~\ref{fig:balls}b, Ball\_1 is the image of the side-view of a 4-color ball where the ball is lit from behind the camera with a white light. The same ball is lit from above with a bluish light, image Ball\_2. The images are in pixel-wise correspondence. We first color correct Ball\_2 to match Ball\_1 by using linear regression. Its results is shown in image Ball\_3 where the colors are incorrectly mapped and the red color segment looks particularly wrong. Ball\_4 shows the results of correcting Ball\_2 by solving for the Homography, visually there is much better color correction.

In this paper, we propose that to map one {\it photometric} view to another we must map the colors correctly independent of shading. Since shading only affects the brightness, or magnitude, of the RGB vectors, it is possible to find the $3\times3$ map which maps the color {\it rays} (the RGBs with arbitrary scalings) in one photometric view to corresponding rays in another. We note that this ''ray matching'' is precisely the circumstance in geometric mapping for co-planar and corresponding points in two images~\cite{Hartley2004}. 
An RGB without shading can be encoded as the rg-chromaticity coordinate. In Figure~\ref{fig:balls}c, the 4 reflectances from the ball correspond to 4 points in an rg-chromaticity diagram which define the quadrilaterals shown in the left and right of the panel (for respectively for the images Ball\_2 and Ball\_1). The mapping between the two chromaticity diagrams is an homography. We show that the color calibration problem -- mapping device RGBs recorded for a color chart to corresponding XYZs -- can be formulated as a homography problem. 

In real images of a color chart the shading can vary from one side to the other and not accounting for this shading variation can result in an incorrect calculation of the color correction matrix~\cite{FuntIntensity}. Methods exist for solving for color correction independent of shading including minimizing the singular fitting error~\cite{FuntIntensity} or an alternating least-squares (ALS) approach where one successively and iteratively solves for the color correction metric and the shading~\cite{ALS}.

According to this paper the correct formalism for color correction independent of shading is to solve for the homography relating the image colors of a color chart to, for example, corresponding XYZ measurements. Unlike linear methods finding the best homography, disadvantageously, requires an iterative approach but advantageously allows different error metrics to be used (not just least-squares error). Compared to respectively, simple least-squares and the prior-art ALS method, homography-based correction delivers a 40 and 10\% performance increment.
\section{Background}
\label{sec:bg}

For the geometric planar homography problem, we write:
\begin{equation}
\left [
\begin{array}{c}
\alpha x\\
\alpha y\\
\alpha
\end{array}
\right ]^\intercal
=
\left [
\begin{array}{c}
x'\\
y'\\
1
\end{array}
\right ]^\intercal
\left [
\begin{array}{ccc}
h_{11} & h_{12} & h_{13} \\
h_{21} & h_{22} & h_{23} \\
h_{31} & h_{32} & h_{33} 
\end{array}
\right ]
\;\triangleq\;
\underline{x}=H(\underline{x}')
\label{eq:homography}
\end{equation}
In Equation~\ref{eq:homography}, $(x,y)$ and $(x',y')$ denote corresponding image points -- the same physical feature -- in two images. In homogeneous coordinates the vector [$a\; b \;c]$ maps to the coordinates $[a/c \;b/c]$ and so, in Equation~\ref{eq:homography}, the scalar $\alpha$ cancels to form the image coordinate $(x,y)$. For all pairs of corresponding points $(x,y)$ and $(x',y')$ that lie on the same plane in 3 dimensional space, Equation~\ref{eq:homography} exactly characterises the relationship between their images~\cite{Hartley2004}.
To solve for an homography (e.g. for image mosaicking), we need to find at least 4 corresponding points in a pair of images.

\section{Color Homography}
\label{sec:main}

Let us map an RGB $\underline{\rho}$ to a corresponding RGI (red-green-intensity) \underline{c} using a $3\times 3$ matrix $C$:
\begin{equation}
\begin{array}{c}
\underline{\rho}C=\underline{c}\\
\;\\
\left [
\begin{array}{c}
R\\
G\\
B
\end{array}
\right ]^\intercal
\left [
\begin{array}{ccc}
1 & 0 & 1\\
0 & 1 & 1\\
0 & 0 & 1
\end{array}
\right ]
=
\left [ 
\begin{array}{c}
R\\
G\\
R+G+B
\end{array}
\right ]^\intercal

\end{array}
\label{eq:chromaticity_conversion}
\end{equation}

The $r$ and $g$ chromaticity coordinates are written as ${r=R/(R+G+B)} \;,\;{g=G/(R+G+B)}$ interpreting the right-hand-side of Equation~\ref{eq:chromaticity_conversion} as a homogeneous coordinate we see that $
\underline{c}\propto \left [
\begin{array}{ccc} r&g&1
\end{array}
\right]
$. In the following proof it is useful to represent 2-d chromaticities by their corresponding 3-d homogeneous coordinates.

\newtheorem{thm}{Theorem}
\begin{thm}[Color Homography]
Chromaticities across a change in capture condition (light color, shading and imaging device) are a homography apart.
\end{thm}

\begin{proof}
First we assume that across a change in illumination or a change in device where the shading is the same (for the Mondrian-world) the corresponding RGBs are related by a linear transform M. Clearly, $H=CMC^{-1}$ maps colors in RGI form between illuminants. Due to different shading, the RGI triple under a second light is represented as $\underline{c}'=\alpha'\underline{c}H$, where $\alpha'$ denotes the unknown scaling. Without loss of generality let us interpret \underline{c} as a homogeneous coordinate i.e. assume its third component is 1. Then, $[r'\;g']=H([r\;g])$ (chromaticity coordinates are a homography $H()$ apart).
\end{proof}

In geometry, homographies are applied for mapping spatial coordinates in one image to correspondences in another. In color, we apply homography to shading independent color correction -- the homography that maps 3D colors to 3D color matches. As an example we find the 3x3 matrix that best maps the shading dependent RGBs in the image of a color target and map them to the corresponding XYZs in a way that is independent of the shading. We can visualize solving for the homography as finding the best linear relation that maps chromaticities in one viewing condition to another. Where by linear means we map 2D chromaticities to corresponding 3D rays, apply the $3\times3$ matrix, and then recompute chromaticities.
\section{Color Homography Color Correction}
\label{sec:exp}

\begin{figure}
\begin{center}
  \includegraphics[width=\linewidth]{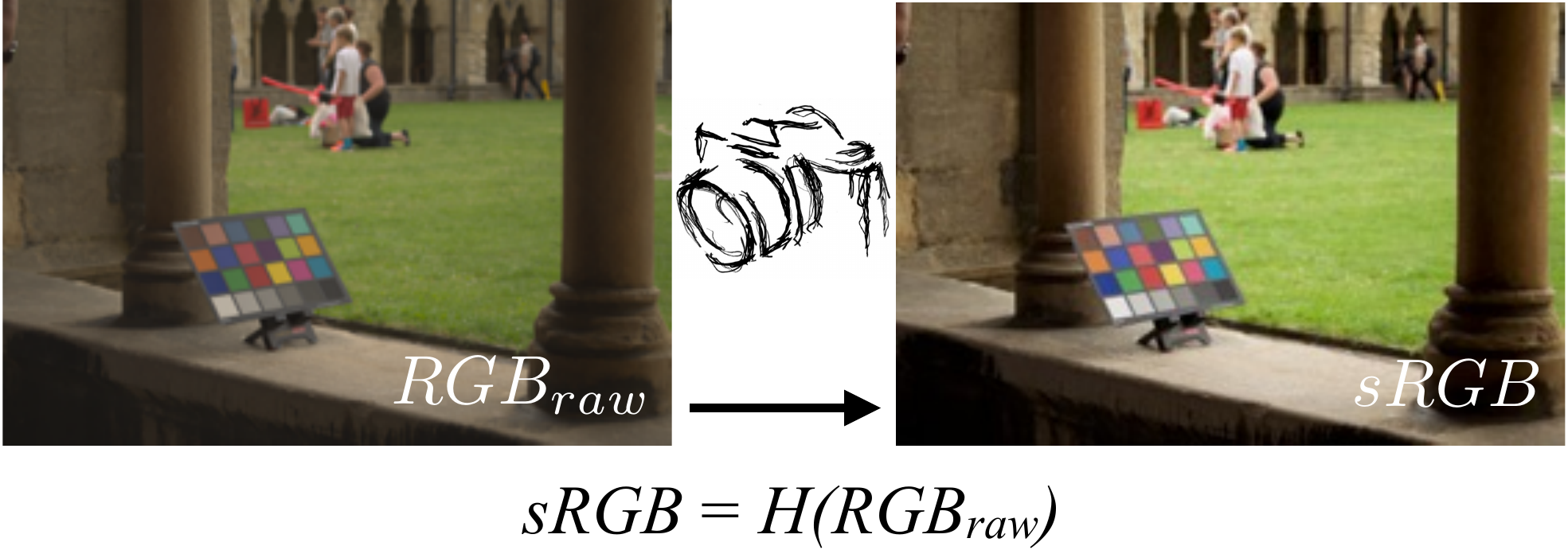}
\end{center}
\caption{Color correction (mapping raw to display sRGB) is an homography problem. This figure is also a color chart example which was used for our color correction evaluation.}
\label{fig:chart}
\end{figure}
%-------------------------------------------------------------------------
In Figure~\ref{fig:chart}, we show the picture of an image in the raw RGB space of a camera and the corresponding reproduction when the colors are corrected for display. Professional imaging scientist might take a picture of a color checker and a second picture of a uniform gray target with same size in the same location. By dividing the RGB image of the checker by the image of the gray-target the shading is removed and then the shading corrected RGBs can be mapped to known reference display color coordinates using simple least-squares. However, this two-step approach is inconvenient and in some cases cannot be done at all (e.g. in an on-going surveillance situation).

Therefore, if a color chart is not used we must solve for the $3\times3$ color correction matrix by solving for the homography.

\subsection{Alternating Least-Squares Color Correction}
Suppose $A$ and $B$ denote respectively $n\times 3$ matrices of $n$ corresponding pixels with respect to a set of captured RGBs and a corresponding set of XYZs. And, due to the relative positions of light and surfaces, the per-pixel shading intensities of the RGB set are usually different. Assuming the Lambertian image formation,
\begin{equation}
DAH\approx B
\label{eq:image_formation}
\end{equation}
where $D$ is an $n\times n$ diagonal matrix of shading factors and $H$ is a $3\times 3$ color correction matrix. In \cite{ALS} this equation is solved by using Alternating Least-Squares (ALS) described in Algorithm~\ref{alg:als}
\begin{algorithm}
\SetAlgoLined
    $i=0$, $A^0=A$\;
    \Repeat{$\norm{A^{i}-A^{i-1}}<\epsilon$}{
    $i=i+1$\;
    $\min_{D^i} \norm{D^iA^{i-1}-B}$\;
    $\min_{H^i} \norm{D^iA^{i-1}H^i-B}$\;
    $A^{i}=D^iA^{i-1}H^i$\;
    }
\caption{Homography from Alternating Least-Square}
\label{alg:als}
\end{algorithm}
where $D^i$ and $H^i$ are iteratively and successively minimized. The effect of the individual $H^i$ and $D^i$ can be combined into a single matrix $D=\prod_{i} D^i$ and $H=\prod_{i} H_i$. The Homography computed here (not because we calculate the shading we find the color correction independent of shading) attempts to minimize a least-squares error. However, as we know the output of color correction might be better assessed with reference to perceived error. For example we might compare the color corrected RGBs with the corresponding XYZs using CIE L*a*b* or CIE L*u*v*~\cite{cie}.

Other prior art~\cite{Funt-angle} searches for the best color correction method and with respect to that framework the practitioner can choose the error metric used. However, because, in this paper, we have made the link to homographies we prove that the apposite tool is to use a particular robust searching technique -- used frequently in geometric computer vision -- called RANSAC.

\subsection{RANSAC Color Correction}
To solve for a homography we need 4 corresponding sets of chromaticities (e.g. the rg and xy chromaticities for the RGBs and XYZs). We might then test how good this homography is for mapping the rg and xy chromaticities. In this paper we are interested in fitting a homography that minimizes a perceptual error and in line with~\cite{RPCC} we first choose CIE L*u*v*. It then makes sense to measure the goodness of fit by comparing actual and fitted u*v* coordinates.

The number of chromaticities that are fitted within a criterion error are said to belong to a consensus set. We iteratively find a homography using match sets of 4 corresponding rg and uv chromaticities and measure the size of the chromaticity error and the size of the consensus set. We keep sampling until the best consensus set is sufficiently large or we have reached the maximum number of trials (in which case we choose the homography that has the lowest error). RANSAC stands for ``random sampling consensus''. A key advantage of RANSAC is that it is a robust statistical estimator~\cite{robust_stat}. However, due to the randomness nature of RANSAC, our color correction does not always give the ``best'' solution. RANSAC also assumes that the majority of data are inliers whose distribution can be explained by a set of model parameters.

\section{Evaluation}
In Figure~\ref{fig:chart} left we show a raw image and the color corrected counterpart (with tone mapping) in the right. We wish to evaluate color correction when shading varies across the color target (in this case a Macbeth Color Checker). We adopt the following experimental methodology. First we measure, in the lab, the actual ground truth XYZs. Second we, in situ, measure the RGBs in a raw image where the shading can vary across the chart. We now solve directly for the correction matrix mapping the RGBs to XYZs. We do this using simple least-squares and using two homography methods. First, we use the Alternating Least-Squares method and second the RANSAC homography method.

We apply the three computed correction matrices (least-squares, ALS and RANSAC) to the RGBs form the checker where the effects of intensity variation has been removed. The intensity variation is removed by dividing by the brightness image of a uniform grey checker images in the same location in the same scene.

In evaluating our algorithms we will use both CIE L*a*b* and CIE L*u*v*. This said it is important to reiterate how we solve for the correction matrices. The Least-squares matrix is found by minimizing the mapping from RGB to XYZ. The ALS method also minimizes the fit to XYZ but finds the best homography (shading independent linear fit). The RANSAC method is a robust fitting strategy which directly minimizes CIE L*u*v* chromaticity error.

For evaluation, we use 3 images that were taken around a local historical site that is popular with amateur and professional photographers alike (e.g. Figure~\ref{fig:chart}).
The mean, median, 95\% quantile and max $\Delta$E errors are reported in Table~\ref{tab:calibration_error} and Table~\ref{tab:calibration_error_luv}. It is clear that RANSAC-homography-based color correction supports a significantly improved color correction performance compared with the simple least square (all error measures are about 40\% improved). Compared with the ALS color correction, all mean errors are improved by about 10\%, and all median errors are improved by about 20\%, at the cost of getting slightly higher 95\% quantile error and maximum error.
\begin{table}[htb]
\centering
\caption{CIE $\Delta$E L*a*b* error}
\label{tab:calibration_error}
\begin{tabular}{lcccc} \toprule
Method & Mean & Median & 95\% & Max \\ \midrule
Least-Squares & 3.70 & 3.30 & 7.73  & 8.39 \\
ALS~\cite{ALS} & 2.36 & 2.04 & \textbf{5.11} & \textbf{5.64} \\
Homography & \textbf{2.18} & \textbf{1.69} & 5.46 & 6.18 \\
\bottomrule
\end{tabular}
\end{table}

\begin{table}[htb]
\caption{CIE $\Delta$E L*u*v* error}
\label{tab:calibration_error_luv}
\centering
\begin{tabular}{lcccc} \toprule
Method & Mean & Median & 95\% & Max \\ \midrule
Least-Squares & 4.28 & 3.84 & 8.30  & 9.53 \\
ALS~\cite{ALS} & 2.73 & 2.41 & \textbf{5.82} & \textbf{6.54} \\
Homography & \textbf{2.58} & \textbf{2.11} & 6.62 & 6.94 \\
\bottomrule
\end{tabular}
\end{table}

\section{Conclusion}
\label{sec:con}
In this paper, we shown the surprising result that colors across a change in viewing condition (changing light color, shading and camera) are related by a homography. Our RANSAC-based color homography color correction is robust to outliers and delivers improved color fidelity compared with the state-of-the-arts.

\section{Acknowledgment}
This work was supported by EPSRC Grant EP/M001768/1.

{\small
\bibliographystyle{plain}
%\bibliography{biblio.bib}

}

\end{document}